\documentclass{article} % For LaTeX2e
\usepackage{PRIMEarxiv}
\errorcontextlines\maxdimen

\usepackage{natbib} 
\usepackage{amsmath,amsfonts,bm}

\def\eqref#1{equation~\ref{#1}}
\def\1{\bm{1}}

\DeclareMathAlphabet{\mathsfit}{\encodingdefault}{\sfdefault}{m}{sl}
\SetMathAlphabet{\mathsfit}{bold}{\encodingdefault}{\sfdefault}{bx}{n}

\usepackage{amsmath, amssymb, amsthm}
\usepackage{graphicx}
\usepackage{hyperref}
\usepackage{url}
\usepackage{bm} 
\usepackage{booktabs} 
\usepackage{multirow}
\usepackage{algorithm}
\usepackage{algorithmic}
\numberwithin{equation}{section}
\usepackage{adjustbox}
\newtheorem{theorem}{Theorem}[section]
\newtheorem{lemma}{Lemma}[section]
\newtheorem{proposition}[theorem]{Proposition}
\newtheorem{corollary}[theorem]{Corollary}
\theoremstyle{definition}
\newtheorem{definition}[theorem]{Definition}

\theoremstyle{remark}

\title{Scalable Random Wavelet Features: Efficient Non-Stationary Kernel Approximation with Convergence Guarantees}
\author{Anonymous Authors}
\author{Sawan Kumar \\
  Department of Applied Mechanics\\
  Indian Institute of Technology Delhi\\
  \texttt{sawan.kumar@am.iitd.ac.in} \\
  \And
  Souvik Chakraborty \\
  Department of Applied Mechanics \\
  Yardi School of Artificial Intelligence (ScAI) \\
  Indian Institute of Technology Delhi \\
  \texttt{souvik@am.iitd.ac.in}
}

\begin{document}
\maketitle

\begin{abstract}
  Modeling non-stationary processes, where statistical properties vary across the input domain, is a critical challenge in machine learning; yet most scalable methods rely on a simplifying assumption of stationarity. This forces a difficult trade-off: use expressive but computationally demanding models like Deep Gaussian Processes, or scalable but limited methods like Random Fourier Features (RFF). We close this gap by introducing Random Wavelet Features (RWF), a framework that constructs scalable, non-stationary kernel approximations by sampling from wavelet families. By harnessing the inherent localization and multi-resolution structure of wavelets, RWF generates an explicit feature map that captures complex, input-dependent patterns. Our framework provides a principled way to generalize RFF to the non-stationary setting and comes with a comprehensive theoretical analysis, including positive definiteness, unbiasedness, and uniform convergence guarantees. We demonstrate empirically on a range of challenging synthetic and real-world datasets that RWF outperforms stationary random features and offers a compelling accuracy-efficiency trade-off against more complex models, unlocking scalable and expressive kernel methods for a broad class of real-world non-stationary problems.
\end{abstract}

\section{Introduction}

The ability to model complex, real-world phenomena is one of the central challenges in machine learning. Domains such as geospatial modeling, where terrain varies drastically across regions, or speech analysis, where signals exhibit bursts of volatility, are often characterized by pronounced \textit{non-stationarity}, meaning their statistical properties change across the input space. Gaussian Processes (GPs) offer a principled framework for such problems, providing robust uncertainty estimates and flexible, non-parametric modeling \citep{williams2006gaussian}. Despite these advantages, exact GPs suffer from two major limitations: their expressivity is often constrained by the choice of kernel, and their computational cost scales cubically with the number of training points, rendering them impractical for modern large-scale applications \citep{liu2020gaussian}.

Most of the current approaches force a trade-off between expressivity and efficiency. On one hand, methods like Random Fourier Features (RFF) achieve impressive scalability by approximating the kernel with a linear-in-data feature map \citep{pmlr-v70-cutajar17a,avron2017random,rahimi2007random}. Yet, their dependence on Bochner's theorem \citep{bochner2005harmonic} fundamentally restricts them to stationary kernels, which assume uniform behavior across the entire domain. Applying these models to non-stationary data leads to systematic mis-specification, resulting in compromised predictive accuracy and uncalibrated uncertainty estimates \citep{cheema2024integrated,hensman2013gaussian,hensman2018variational}. On the other hand, expressive models like Deep GPs \citep{salimbeni2019deep}, spectral mixtures \citep{tompkins2020sparse}, and input-dependent kernels \citep{iddgp2020} can capture non-stationarity, but they often reintroduce prohibitive computational costs, complex inference schemes, and challenges in optimization and hyperparameter tuning. The gap between scalable stationary models and complex non-stationary ones still remains.

In this work, we close this gap by introducing \textbf{Random Wavelet Features (RWF)}, a scalable and expressive framework for non-stationary kernel approximation. Instead of relying on globally supported sinusoidal bases like RFF and its variants, we construct random features from \textit{wavelets} family of functions that are inherently localized in both space and frequency. By sampling wavelets at random scales and shifts, RWF generates an explicit feature map that can adapt to local data characteristics. This multi-resolution structure allows the model to capture sharp, localized events with fine-scale wavelets while simultaneously modeling smooth, long-range trends with coarse-scale wavelets. The result is a principled method that generalizes random features to the non-stationary setting while preserving the linear-time complexity that makes them elegant and efficient.
% Our main contributions are:
Our main contributions are threefold. First, we provide a comprehensive theoretical analysis of RWF, including positive definiteness of the induced kernels, unbiasedness and variance bounds, and uniform convergence guarantees with explicit sample complexity. Second, we show that RWF achieves $\mathcal{O}(ND^2)$ training complexity, retaining the scalability of random feature methods while directly encoding non-stationarity through wavelet localization. Finally, we demonstrate empirically on synthetic, speech, and large-scale regression benchmarks that RWF consistently improves upon stationary random features and offers atleast competitive accuracy–efficiency trade-off against more complex non-stationary models.

\subsection{Related Work}

\textbf{Scalable Kernel Approximations.}
The random features framework was pioneered by \citep{rahimi2007random}, showing shift-invariant kernels can be approximated using random Fourier features with linear-time computations. This framework has since been extended in several directions, including computationally efficient variants such as Fast kernel learning \citep{wilson2014fast}, theoretical guarantees on approximation error \citep{sriperumbudur2015optimal, avron2017random,li2021towards}, and structured sampling schemes \citep{choromanski2017structured}. There are works that extend the random Fourier features beyond classical settings using variational approximations \citep{hensman2018variational}, adaptive feature learning \citep{zhen2020learning,shi2024adaptive}, and even a connection to quantum machine learning \citep{landman2022classically}. Recent progress extends spectral approximations to capture a wider spectrum of kernel families, thereby enhancing the expressivity of scalable feature maps \citep{langrene2024spectral}. While these methods achieve scalability, their reliance on stationary Fourier bases limits their ability to capture non-stationary or localized phenomena, which are crucial in many scientific domains.

\textbf{Wavelet-motivated approximations.}
Wavelets have previously been used for kernel design through
wavelet support vector machines and wavelet kernel learning
\citep{zhang2004wavelet,yger2011wavelet}, where kernels are derived analytically, or wavelet transforms are used as preprocessing.
More recently, \citep{guo2024wavelet} proposed a Bayesian kernel model
based on fixed wavelet bases for high-dimensional Bayesian linear regression.
While these approaches illustrate the value of wavelets for capturing
local structure, they rely on fixed or predefined wavelets
dictionaries and do not provide scalable Monte Carlo approximations or
theoretical guarantees such as unbiasedness or uniform convergence.

\textbf{Hybrid and Modern Kernel Learning.}
Several approaches have been developed to capture non-stationarity in Gaussian processes.
Spectral mixture kernels \citep{wilson2013gaussian,remes2017non} extend the idea of stationary spectral representations to capture richer structures,
while deep Gaussian processes \citep{damianou2013deep,salimbeni2019deep} enhance flexibility through hierarchical composition.
Despite their expressivity, these models typically retain high complexity in the number of training points.
Beyond these directions, some work has explored combining scalable approximations with expressive models.
Kernel interpolation methods (KISS-GP) \citep{wilson2015kernel} exploit structure in inducing-point methods,
while deep kernel learning \citep{wilson2016deep} integrates neural networks with Gaussian processes. Recently, these ideas have been pushed further with deep random features for spatiotemporal modeling \citep{chen2024deep}, graph-based random Fourier features for scalable GP \citep{zhang2025graph}, and compositional kernels in adaptive RKHS representation \citep{shi2024adaptive}.
Despite these advances, existing methods often trade off scalability, expressivity, and interpretability. Our work is positioned at this intersection where we aim to design feature maps that inherit the scalability of random features while enabling flexible, non-stationary modeling.

\section{Preliminaries and Background}
A brief review of Gaussian Process regression (GPR), sparse variational GPs, and random-feature GPs is provided to ground our wavelet construction.

\subsection{Gaussian Processes}
Given training inputs \(\mathbf{X}=[\bm{x}_1,\dots,\bm{x}_N]^\top\in\mathbb{R}^{N\times d}\) and targets \(\bm{y}\in\mathbb{R}^N\), we consider a Gaussian process prior over a latent function \(f\). The observations \(y_n\) are assumed to be noisy evaluations of this function at the corresponding inputs \(\bm{x}_n\):
\begin{equation}
  f \sim \mathcal{GP}(0,k), \qquad y_n = f(\bm{x}_n) + \varepsilon_n, \quad \text{where} \quad \varepsilon_n \sim \mathcal{N}(0, \sigma^2).
\end{equation}
We define the covariance matrix \(\mathbf{K}_{XX}\in\mathbb{R}^{N\times N}\) with \([\mathbf{K}_{XX}]_{ij}=k(\bm{x}_i,\bm{x}_j)\).
The log marginal likelihood, used for training the hyperparameters of GP, is given by the following expression:
\begin{equation}
  \log p(\bm{y}\mid \mathbf{X}) = -\tfrac{1}{2}\bm{y}^\top (\mathbf{K}_{XX}+\sigma^2 \mathbf{I}_N)^{-1}\bm{y}
  - \tfrac{1}{2}\log\det(\mathbf{K}_{XX}+\sigma^2 \mathbf{I}_N) - \tfrac{N}{2}\log(2\pi).
\end{equation}
For a test input \(\bm{x}_*\), let \(\bm{k}_{*X}=[k(\bm{x}_*,\bm{x}_1),\dots,k(\bm{x}_*,\bm{x}_N)]\), \(k_{**}=k(\bm{x}_*,\bm{x}_*)\), and
\begin{equation}
  \bm{\alpha} = (\mathbf{K}_{XX}+\sigma^2 \mathbf{I}_N)^{-1}\bm{y}.
\end{equation}
The predictive posterior moments for test input \(\bm{x}_*\) takes the following form:
\begin{subequations}
  \begin{align}
    \mu_*(\bm{x}_*)      & = \bm{k}_{*X}\bm{\alpha},                                                      \\
    \sigma^2_*(\bm{x}_*) & = k_{**} - \bm{k}_{*X}(\mathbf{K}_{XX}+\sigma^2 \mathbf{I}_N)^{-1}\bm{k}_{X*},
  \end{align}
\end{subequations}
with \(\bm{k}_{X*}=\bm{k}_{*X}^\top\).
The key bottleneck of exact inference is its computational costs \(O(N^3)\) time, and \(O(N^2)\) memory. To address these challenges, several approaches have been introduced in the literature; the most common ones are the sparse approximation of GP.

\textbf{Stochastic Variational GPs (SVGP).}
In SVGP, we introduce $M_u$, inducing inputs 
$\mathbf{Z}_u=[\bm{z}_1,\dots,\bm{z}_{M_u}]^\top$ and inducing variables 
$\bm{u}=f(\mathbf{Z}_u)$ equipped with the prior 
$p(\bm{u})=\mathcal{N}(\bm{0},\mathbf{K}_{uu})$, where 
$[\mathbf{K}_{uu}]_{ij}=k(\bm{z}_i,\bm{z}_j)$.
Defining 
$\mathbf{K}_{fu}\in\mathbb{R}^{N\times M_u}$ with 
$[\mathbf{K}_{fu}]_{nm}=k(\bm{x}_n,\bm{z}_m)$ and 
$\mathbf{Q}_{ff}=\mathbf{K}_{fu}\mathbf{K}_{uu}^{-1}\mathbf{K}_{uf}$, the conditional prior becomes
\begin{equation}
p(\bm{f}\mid\bm{u})
  =\mathcal{N}\!\big(\mathbf{K}_{fu}\mathbf{K}_{uu}^{-1}\bm{u},\;
                     \mathbf{K}_{ff}-\mathbf{Q}_{ff}\big).
\end{equation}

A Gaussian variational posterior $q(\bm{u})=\mathcal{N}(\bm{m},\mathbf{S})$
induces 
$q(\bm{f})=\mathcal{N}(\bm{\mu},\mathbf{\Sigma})$, 
with 
$\bm{\mu}=\mathbf{A}\bm{m}$ and  
$\mathbf{\Sigma}=\mathbf{K}_{ff}-\mathbf{Q}_{ff}+\mathbf{A}\mathbf{S}\mathbf{A}^\top$, 
where $\mathbf{A}=\mathbf{K}_{fu}\mathbf{K}_{uu}^{-1}$.
Under a Gaussian likelihood 
$p(y_n\!\mid f_n)=\mathcal{N}(y_n\!\mid f_n,\sigma^2)$,  
the ELBO simplifies to
\begin{equation}
\mathcal{L}
  =\sum_{n=1}^N 
     \mathbb{E}_{q(f_n)}[\log p(y_n\!\mid f_n)]
   -\mathrm{KL}\!\left(q(\bm{u})\,\Vert\,p(\bm{u})\right),
\end{equation}
where 
$\mathbb{E}_{q(f_n)}[\log p(y_n\!\mid f_n)]
 =-\tfrac{1}{2}\sigma^{-2}\big[(y_n-\mu_n)^2+\Sigma_{nn}\big]
  -\tfrac{1}{2}\log(2\pi\sigma^2)$.

Using a minibatch $\mathcal{B}$ of size $b$ gives the unbiased estimator  
$\widehat{\mathcal{L}}=(N/b)\sum_{n\in\mathcal{B}}
\mathbb{E}_{q(f_n)}[\log p(y_n\!\mid f_n)]
-\mathrm{KL}(q(\bm{u})\Vert p(\bm{u}))$,  
with per-iteration complexity $O(bM_u^2)$ plus a one-time  
$O(M_u^3)$ factorization of $\mathbf{K}_{uu}$.

Predictive moments at a test point $\bm{x}_*$ follow the closed-form GP equations:
$\mu_*(\bm{x}_*)=\bm{k}_{*u}\mathbf{K}_{uu}^{-1}\bm{m}$ and 
$\sigma_*^2(\bm{x}_*)=k_{**}-\bm{k}_{*u}\mathbf{K}_{uu}^{-1}\bm{k}_{u*}
+\bm{k}_{*u}\mathbf{K}_{uu}^{-1}\mathbf{S}\mathbf{K}_{uu}^{-1}\bm{k}_{u*}$,
where  
$\bm{k}_{*u}=[k(\bm{x}_*,\bm{z}_1),\dots,k(\bm{x}_*,\bm{z}_{M_u})]$.

\subsection{Random Fourier Feature GPs (RFF-GP)}
The random Fourier features approach introduced by \citet{rahimi2007random}
approximates stationary kernels using explicit feature maps.
Consider a zero-mean Gaussian process $f \sim \mathcal{GP}(0,k)$ with a stationary kernel
$k(\bm{x},\bm{x}') = k(\bm{x}-\bm{x}')$.
By Bochner’s theorem \citep{bochner2005harmonic}, the kernel admits the spectral representation
\begin{equation}
  k(\bm{x},\bm{x}') = \int_{\mathbb{R}^d} e^{i\bm{\omega}^\top(\bm{x}-\bm{x}')} p(\bm{\omega}) \, d\bm{\omega},
\end{equation}
where $p(\bm{\omega})$ is the normalized spectral density of kernel $k$.
Introducing a random phase $b \sim \mathrm{Unif}[0,2\pi]$, this can be expressed as an expectation
over cosine features:
\begin{equation}
  k(\bm{x},\bm{x}') = \mathbb{E}_{\bm{\omega},b}\!\left[2\cos(\bm{\omega}^\top\bm{x}+b)\cos(\bm{\omega}^\top\bm{x}'+b)\right].
\end{equation}

Approximating the expectation with $D$ Monte Carlo samples
$\{(\bm{\omega}_j,b_j)\}_{j=1}^D$ yields the random feature map $\bm{z}:\mathcal{X}\to\mathbb{R}^D$,
\begin{equation}
  \bm{z}(\bm{x}) = \frac{1}{\sqrt{D}}\big[\phi_1(\bm{x}),\ldots,\phi_D(\bm{x})\big]^\top,
  \qquad \phi_j(\bm{x}) = \sqrt{2}\cos(\bm{\omega}_j^\top \bm{x} + b_j),
\end{equation}
such that the approximate kernel is $\hat{k}(\bm{x},\bm{x}') = \bm{z}(\bm{x})^\top \bm{z}(\bm{x}')$.

From the GP perspective, this corresponds to replacing the infinite-dimensional feature space
with the finite-dimensional features $\bm{z}(\cdot)$, leading to a Bayesian linear regression model.
Placing a Gaussian prior $\bm{w} \sim \mathcal{N}(\bm{0}, \mathbf{I}_D)$ on the weights, the Gaussian posterior with covariance and mean given by,
\begin{subequations}
  \begin{align}
    \mathbf{S}_{\bm{w}} & = \big(\mathbf{I}_D + \sigma^{-2}\mathbf{Z}^\top \mathbf{Z}\big)^{-1}, \\
    \bm{m}_{\bm{w}}     & = \sigma^{-2}\mathbf{S}_{\bm{w}} \mathbf{Z}^\top \bm{y},
  \end{align}
\end{subequations}
where $\mathbf{Z}\in\mathbb{R}^{N\times D}$ collects the feature maps of the training inputs.
The Gaussian predictive distribution for a new test point $\bm{x}_*$ has the mean and covariance defined as,
\begin{subequations}
  \begin{align}
    \mu_*(\bm{x}_*)                   & = \bm{z}(\bm{x}_*)^\top \bm{m}_{\bm{w}},                                 \\
    \mathrm{Var}[y_*\mid \mathcal{D}] & = \bm{z}(\bm{x}_*)^\top \mathbf{S}_{\bm{w}} \bm{z}(\bm{x}_*) + \sigma^2.
  \end{align}
\end{subequations}

The RFF-GP framework is thus a scalable approximation for stationary kernels. However, its reliance on globally supported Fourier features limits its ability to model non-stationarity. For further details on RFF and examples, see Appendix~\ref{app:rf_gp}

\section{Proposed methodology}
Random Fourier-based kernel approximation methods exploit Bochner's theorem \citep{rahimi2007random} to yield scalable approximations for stationary kernels, but they struggle to capture non-stationarity. Sparse variational GPs model non-stationarity with expressive kernels yet rely on inducing sets and cubic costs in \(M_u\) per update. We propose Random Wavelet Features (RWF), which construct non-stationary kernels via multi-resolution, locally supported wavelets. By sampling wavelet scales and shifts, RWF provides an explicit feature map \(z(\cdot)\) that: (i) induces a positive definite non-stationary kernel; (ii) preserves linear-time training and prediction as in RFF-GPs; and (iii) captures localized, multi-resolution structure that stationary RFF lacks.

\subsection{Wavelet-based Kernel Construction}
To model non-stationarity, a kernel's properties must adapt across the input domain. Stationary kernels, often approximated by Random Fourier Features (RFF), rely on globally supported sinusoidal bases that are inherently spatially invariant. In contrast, wavelets offer a natural alternative by providing a basis that is localized in both space and frequency. By randomizing the scale (controlling frequency) and shift (controlling spatial location) of wavelet atoms, we can construct a flexible, non-stationary kernel.

Our construction begins with a mother wavelet \(\psi: \mathbb{R}^d \to \mathbb{R}\), a function with zero mean and unit \(L^2\) norm (see Appendix~\ref{app:wavelet_prelim} for details). From \(\psi\), we generate a family of wavelet atoms via isotropic scaling and translation:
\begin{equation}
  \psi_{s,\bm{t}}(\bm{x}) = s^{-d/2} \psi\left(\frac{\bm{x}-\bm{t}}{s}\right), \quad \text{for scale } s>0 \text{ and shift } \bm{t} \in \mathbb{R}^d.
\end{equation}
Each atom \(\psi_{s,\bm{t}}\) is a localized ``wave packet'' centered at \(\bm{t}\) with spatial extent proportional to \(s\). Let \(\Theta = (0, \infty) \times \mathbb{R}^d\) be the parameter space of scales and shifts. We define a non-stationary kernel by integrating over this space with respect to a non-negative measure \(\mu(ds\,d\bm{t})\):
\begin{equation}
  k(\bm{x}, \bm{y}) = \int_{\Theta} \psi_{s,\bm{t}}(\bm{x}) \psi_{s,\bm{t}}(\bm{y}) \, \mu(ds\,d\bm{t}).
  \label{eq:wavelet_kernel_measure}
\end{equation}
This construction guarantees positive definiteness, as the integrand is a product of scalar features. If \(\mu\) has a density \(p(s, \bm{t}) \ge 0\), the kernel becomes:
\begin{equation}
  k(\bm{x}, \bm{y}) = \int_{0}^{\infty} \int_{\mathbb{R}^d} \psi_{s,\bm{t}}(\bm{x}) \psi_{s,\bm{t}}(\bm{y}) \, p(s, \bm{t}) \, d\bm{t} \, ds.
  \label{eq:wavelet_kernel_density}
\end{equation}
The density \(p(s, \bm{t})\) governs the kernel's properties. A common choice is a factorized form \(p(s, \bm{t}) = p_s(s) p_t(\bm{t})\), where \(p_s\) (e.g., log-uniform) spans multiple resolutions and \(p_t\) (e.g., uniform over the data's convex hull) provides spatial coverage.

\subsection{Random Wavelet Feature Sampling Strategy}
The integral in \eqref{eq:wavelet_kernel_density} is typically intractable. We approximate it via Monte Carlo sampling, which forms the basis of our random features.
\begin{definition}[Random Wavelet Features]
  Sample \((s_i, \bm{t}_i)_{i=1}^D\) i.i.d.~from a distribution with density \(p(s, \bm{t})\) and define the random feature map \(z: \mathbb{R}^d \to \mathbb{R}^D\) as:
  \begin{equation}
    z(\bm{x}) = \frac{1}{\sqrt{D}} \left[ \psi_{s_1,\bm{t}_1}(\bm{x}), \dots, \psi_{s_D,\bm{t}_D}(\bm{x}) \right]^\top.
  \end{equation}
  The corresponding kernel approximation is \(\hat{k}(\bm{x}, \bm{y}) = z(\bm{x})^\top z(\bm{y})\).
\end{definition}
By construction, \(\hat{k}(\bm{x}, \bm{y})\) is an unbiased estimator of \(k(\bm{x}, \bm{y})\). This formulation transforms the kernel method into a Bayesian linear model, enabling efficient training and prediction. The full procedure is detailed in Algorithm~\ref{alg:RWF-gp}.

\begin{algorithm}[h]
  \caption{RWF-GP Training and Prediction}
  \label{alg:RWF-gp}
  \begin{algorithmic}[1]
    \STATE \textbf{Input:} Training data \((\mathbf{X}, \bm{y})\), test inputs \(\mathbf{X}_*\), number of features \(D\), wavelet \(\psi\), sampling distribution \(p(s,\bm{t})\).
    \STATE \textbf{Hyperparameters:} Noise variance \(\sigma^2\), parameters of \(p(s,\bm{t})\).
    \STATE \textbf{Training:}
    \STATE Sample \((s_i, \bm{t}_i) \sim p(s,\bm{t})\) for \(i=1, \dots, D\).
    \STATE Construct feature matrix \(Z \in \mathbb{R}^{N \times D}\) where \(Z_{ni} = \frac{1}{\sqrt{D}}\psi_{s_i,\bm{t}_i}(\bm{x}_n)\).
    \STATE Compute weight posterior: \(S_{\bm{w}} = (I_D + \sigma^{-2} Z^\top Z)^{-1}\) and \(\bm{m}_{\bm{w}} = \sigma^{-2} S_{\bm{w}} Z^\top \bm{y}\).
    \STATE Optimize hyperparameters (e.g., \(\sigma^2\), params of \(p\)) by maximizing the marginal likelihood of the Bayesian linear model.
    \STATE \textbf{Prediction:}
    \STATE Construct test feature matrix \(Z_* \in \mathbb{R}^{N_* \times D}\) where \([Z_*]_{ji} = \frac{1}{\sqrt{D}}\psi_{s_i,\bm{t}_i}(\bm{x}_{*,j})\).
    \STATE Compute predictive mean: \(\bm{\mu}_* = Z_* \bm{m}_{\bm{w}}\).
    \STATE Compute predictive variance: \(\bm{\sigma}^2_* = \text{diag}(Z_* S_{\bm{w}} Z_*^\top) + \sigma^2\).
    \STATE \textbf{Output:} Predictive distribution \(\mathcal{N}(\bm{\mu}_*, \bm{\sigma}^2_*)\).
  \end{algorithmic}
\end{algorithm}

\subsection{Practical Considerations}
\textbf{Computational Complexity.}
RWF is efficient because computational cost scales linearly with the dataset size. Constructing $D$ random wavelet features over $N$ inputs of dimension $d$ costs 
$\mathcal{O}(NDd)$, after which training reduces to the primal form of GP regression in a $D$-dimensional feature space. Forming $Z^\top Z$ requires $\mathcal{O}(ND^2)$ and the resulting $D\times D$ system is solved in $\mathcal{O}(D^3)$, so for $N\!\gg\!D$ the 
overall training cost is dominated by $\mathcal{O}(ND^2)$. Predictions require 
$\mathcal{O}(D^2)$ per test point. In contrast, Exact GPs scale as $\mathcal{O}(N^3)$ and SVGP incurs $\mathcal{O}(NM^2)$ \emph{per} optimization step due to iterative variational updates. RWF computes its posterior in a single closed-form solve, yielding substantial wall-clock speedups for large-scale non-stationary learning.

The key to modeling non-stationarity lies in the practical choices for the wavelet family and sampling distribution. The choice of mother wavelet \(\psi\) (e.g., Morlet for time-frequency analysis or Daubechies for sharp transitions) and the sampling distribution \(p(s, \bm{t})\) (e.g., log-uniform for scales, uniform for shifts) \citep{bergstra2012random,jeffreys1946invariant} allows the model to adapt to multi-resolution signal structures. For stable training, it is beneficial to regularize the model by constraining the sampling range for scales and applying weight decay to the linear model.

\section{Theoretical Analysis}

To analyze the quality of our approximation, we establish uniform convergence guarantees. Our analysis relies on bounding the complexity of the function class induced by the wavelet features. We define the following key quantities: \(B = \sup_{s,t,\bm{x}} |\psi_{s,t}(\bm{x})|\) as the uniform bound on the feature magnitude, and \(K = \sup_{\bm{x}} k(\bm{x},\bm{x}')\) as the maximum kernel value.

\subsection{Positive Definiteness of Wavelet Kernels}

\begin{theorem}[Positive Definiteness of Wavelet-Based Kernels]\label{thm:kernel_app}
  Let \(\psi: \mathbb{R}^d \to \mathbb{R}\) be a mother wavelet function, and define the family of wavelets as \(\psi_{s,t}(\bm{x}) = s^{-d/2} \psi\left( s^{-1} (\bm{x} - t) \right)\) for scale \(s > 0\) and translation \(t \in \mathbb{R}^d\). Let \(p(s, t): \mathbb{R}^+ \times \mathbb{R}^d \to [0, \infty)\) be a non-negative measure such that the integral is well-defined and finite for all \(\bm{x}, \bm{y} \in \mathbb{R}^d\). Then, the function
  \begin{equation}
    k(\bm{x}, \bm{y}) = \int_{\mathbb{R}^+} \int_{\mathbb{R}^d} \psi_{s,t}(\bm{x}) \psi_{s,t}(\bm{y}) \, p(s, t) \, dt \, ds
  \end{equation}
  is a positive definite kernel on \(\mathbb{R}^d \times \mathbb{R}^d\).
\end{theorem}
(Proof in Appendix~\ref{app:proof_kernel_app}.)

\subsection{Unbiasedness and Variance Bounds}

\begin{lemma}[Unbiasedness]\label{lem:unbiased}
  For all \(\bm{x},\bm{y}\in\mathcal{X}\), the wavelet random feature approximation is unbiased: \(\mathbb{E}[\hat k(\bm{x},\bm{y})]=k(\bm{x},\bm{y})\).
\end{lemma}
(Proof in Appendix~\ref{app:proof_unbiased}.)

\begin{lemma}[Variance Bound]\label{lem:variance}
  For all \(\bm{x},\bm{y}\in\mathcal{X}\), the variance of the approximation is bounded:
  \begin{equation}
    \mathrm{Var}\left[\hat k(\bm{x},\bm{y})\right]\le\frac{B^2}{D}.
  \end{equation}
\end{lemma}
(Proof in Appendix~\ref{app:proof_variance}.)

\subsection{Uniform Convergence Guarantees}

\begin{theorem}[Uniform Convergence of Random Wavelet Features]\label{thm:uniform_conv}
  Let \(\mathcal{M} \subset \mathbb{R}^d\) be a compact set with diameter \(\operatorname{diam}(\mathcal{M})\). Let \(k(\bm{x},\bm{y})\) be a positive definite kernel as in Theorem~\ref{thm:kernel_app}, and define the random feature map \(z: \mathbb{R}^d \to \mathbb{R}^D\) by independently sampling \((s_i, t_i) \sim p\) for \(i=1,\dots,D\) and setting
  \begin{equation}
    z(\bm{x}) = \frac{1}{\sqrt{D}} \left[ \psi_{s_1,t_1}(\bm{x}), \dots, \psi_{s_D,t_D}(\bm{x}) \right]^\top.
  \end{equation}
  Assume \(k\) and the feature map are Lipschitz continuous with constants \(L_k\) and \(L_z\), respectively. Then, for any \(\epsilon > 0\),
  \begin{equation}
    \mathrm{Pr} \left[ \sup_{\bm{x},\bm{y} \in \mathcal{M}} |z(\bm{x})^\top z(\bm{y}) - k(\bm{x},\bm{y})| \ge \epsilon \right] \le 2 \left( \frac{4 \operatorname{diam}(\mathcal{M}) L_z}{\epsilon} \right)^{2d} \exp\left( -\frac{D \epsilon^2}{8 B^2} \right).
  \end{equation}
  (Proof in Appendix~\ref{app:proof_uniform}.)
\end{theorem}

\subsection{Sample Complexity Analysis}

The above theorem provides explicit sample complexity bounds. To achieve approximation error \(\epsilon\) with probability at least \(1-\delta\), it suffices to choose
\begin{equation}
  D \geq \frac{8B^2}{\epsilon^2} \left( 2d \log\left(\frac{4 \operatorname{diam}(\mathcal{M}) L_z}{\epsilon}\right) + \log\left(\frac{2}{\delta}\right) \right).
\end{equation}
This result is derived by inverting the probability bound in Theorem~\ref{thm:uniform_conv}. The constants \(B\) and \(L_z\) depend on the choice of wavelet and are discussed in Appendix~\ref{app:wavelet_prelim}.
This shows that the number of required features scales logarithmically with the desired accuracy and confidence level.

\section{Experiments}
This section presents experiments designed to evaluate the performance of proposed approach. We begin by examining the approximation quality of our approach on non-stationary synthetic data, and then proceed to evaluate it on a highly non-stationary speech signal dataset and benchmark regression tasks, comparing it to various baseline models. Further details about all the experiments can be found in Appendix \ref{app:exp_details}. 

\textbf{Baselines}: We compare against scalable and/or expressive variants: SVGP \citep{hensman2013gaussian}, RFF-GP \citep{rahimi2007random}, Deep GPs \citep{salimbeni2019deep}, and exact GPs (when feasible) as well as specialized GP for non-stationary data: Spectral Mixture kernels \citep{langrene2024spectral}, DRF \citep{chen2024deep}, IDD-GP \citep{iddgp2020}, and Adaptive RKHS Fourier Feature GPs \citep{shi2024adaptive}.

\subsection{Evaluation on synthetic data}
We first evaluate RWF on a non-stationary multi-step function, a setting where shallow GPs with stationary kernels fail to capture input-dependent variations \citep{iddgp2020}. Deep GPs, although offer more expressiveness, struggle with sharp discontinuities. In contrast, RWF enables shallow GPs to fit accurately: Figure~\ref{fig: synthetic_compare_models} shows that RWF-GP captures the non-stationary structure, whereas baselines yield overly smooth or oscillatory fits due to limited kernel flexibility. Table \ref{tab:syn_result} illustrates the superior performance of the proposed approach, both in terms of accuracy and training time, over its competitors.  
Figure~\ref{fig:scalability_combined} summarizes wall-clock time and memory footprints for the compared methods, illustrating the scalability of the proposed approach. Ablation study illustrating the convergence of the proposed approach with feature size
is shown in Appendix \ref{app: syn_dataset_wrf_rff}.

\begin{figure}
  \centering
  \includegraphics[width=\linewidth]{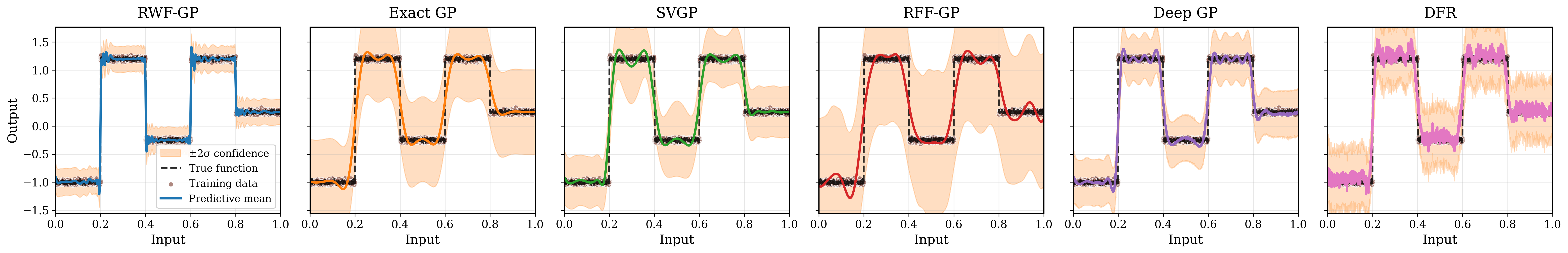}
  \caption{
    Predictive performance of different GP methods on a step function regression task.
    Each panel shows the predictive mean (solid line) with $\pm 2\sigma$ confidence intervals (shaded), training data (dots).
    \textbf{RWF-GP} (ours) captures the discontinuities sharply while maintaining calibrated uncertainty.
    In contrast, {Exact GP}, {Sparse Variational GP}, and {RFF-GP} struggle with sharp transitions, either oversmoothing or miscalibrating the uncertainty.
  }
  \label{fig: synthetic_compare_models}
\end{figure}

\begin{figure}[ht!]
  \centering
  \includegraphics[width=0.85\columnwidth]{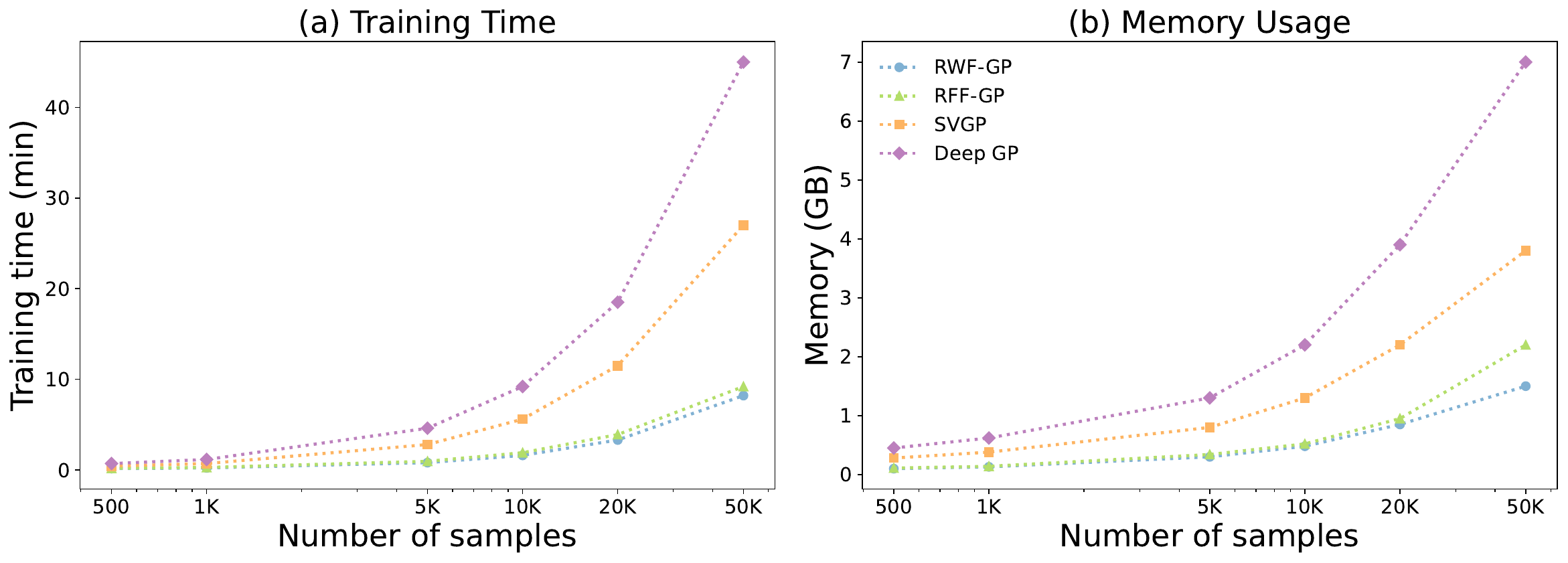}
  \caption{Scalability on the multi-step function. Time and memory vs. number of training samples on the multi-step function: RWF is most efficient; SVGP and Deep GP incur higher cost.}
  \label{fig:scalability_combined}
  \vspace{-2mm}
\end{figure}

\begin{table}[ht!]
\centering
\scriptsize
\caption{Performance comparison of GP baselines on the multi-step function over five runs 
(mean $\pm$ std; lower is better). \textbf{Bold} indicates the best result, and 
\underline{underline} indicates the second best. Methods: Exact = Exact GP, 
SVGP = Stochastic Variational GP, RFF = Random Fourier Features, DRF = Deep-RF GP, 
DGP = Deep GP, SM = Spectral Mixture GP, IDD = Inter-domain Deep GP, 
A-RKHS = Adaptive RKHS GP. Results for SM, IDD, and A-RKHS are from \cite{shi2024adaptive}.}
\label{tab:syn_result}
\resizebox{\textwidth}{!}{%
\begin{tabular}{lccccccccc}
\toprule
 & Exact & SVGP & RFF & DRF & DGP & SM & IDD & A-RKHS & RWF (Ours)\\
\midrule
RMSE & 
\begin{tabular}[c]{@{}c@{}}0.190 \\ $\pm$0.091\end{tabular} &
\begin{tabular}[c]{@{}c@{}}0.231 \\ $\pm$0.014\end{tabular} &
\begin{tabular}[c]{@{}c@{}}0.246 \\ $\pm$0.142\end{tabular} &
\begin{tabular}[c]{@{}c@{}}0.190 \\ $\pm$0.120\end{tabular} &
\begin{tabular}[c]{@{}c@{}}0.162 \\ $\pm$0.110\end{tabular} &
\begin{tabular}[c]{@{}c@{}}0.210 \\ $\pm$0.085\end{tabular} &
\begin{tabular}[c]{@{}c@{}}0.107 \\ $\pm$0.050\end{tabular} &
\begin{tabular}[c]{@{}c@{}}\underline{0.095} \\ \underline{$\pm$0.045}\end{tabular} &
\begin{tabular}[c]{@{}c@{}}\textbf{0.071} \\ \textbf{$\pm$0.011}\end{tabular} \\
CRPS &
\begin{tabular}[c]{@{}c@{}}0.215 \\ $\pm$0.030\end{tabular} &
\begin{tabular}[c]{@{}c@{}}0.392 \\ $\pm$0.025\end{tabular} &
\begin{tabular}[c]{@{}c@{}}0.238 \\ $\pm$0.041\end{tabular} &
\begin{tabular}[c]{@{}c@{}}0.205 \\ $\pm$0.032\end{tabular} &
\begin{tabular}[c]{@{}c@{}}0.187 \\ $\pm$0.028\end{tabular} &
\begin{tabular}[c]{@{}c@{}}0.201 \\ $\pm$0.030\end{tabular} &
\begin{tabular}[c]{@{}c@{}}0.143 \\ $\pm$0.020\end{tabular} &
\begin{tabular}[c]{@{}c@{}}\underline{0.131} \\ \underline{$\pm$0.018}\end{tabular} &
\begin{tabular}[c]{@{}c@{}}\textbf{0.112} \\ \textbf{$\pm$0.010}\end{tabular} \\
NLL &
\begin{tabular}[c]{@{}c@{}}0.042 \\ $\pm$0.012\end{tabular} &
\begin{tabular}[c]{@{}c@{}}0.123 \\ $\pm$0.018\end{tabular} &
\begin{tabular}[c]{@{}c@{}}0.118 \\ $\pm$0.181\end{tabular} &
\begin{tabular}[c]{@{}c@{}}-0.018 \\ $\pm$0.216\end{tabular} &
\begin{tabular}[c]{@{}c@{}}-0.268 \\ $\pm$0.211\end{tabular} &
\begin{tabular}[c]{@{}c@{}}0.220 \\ $\pm$0.180\end{tabular} &
\begin{tabular}[c]{@{}c@{}}-0.820 \\ $\pm$0.080\end{tabular} &
\begin{tabular}[c]{@{}c@{}}\underline{-1.210} \\ \underline{$\pm$0.075}\end{tabular} &
\begin{tabular}[c]{@{}c@{}}\textbf{-1.879} \\ \textbf{$\pm$0.061}\end{tabular} \\
Time  & 12 & 15 & \underline{11} & 18 & 20 & 17 & 17 & \underline{11} & \textbf{9} \\
\bottomrule
\end{tabular}}
\end{table}

\subsection{TIMIT Speech Signal}
We evaluate our approach on a regression task derived from the TIMIT corpus, following prior GP-based studies \citep{shi2024adaptive}. TIMIT poses challenge due to strong non-stationarities in the audio signal, such as localized consonant bursts and slowly varying regions. Models relying on stationary kernels struggle to capture these variations without either over-smoothing or requiring a large number of features. 
Unlike RFF, RWF allocates resolution adaptively: small scales capture sharp attacks and large scales capture smooth regions, thus reducing approximation variance for a fixed feature size.
\textbf{Results:}
Table~\ref{tab:timit_regression_results} reports RMSE and training time. RWF-GP achieves the lowest error compared to the baselines. RFF-GP performs worst with the same number of features \(D\), reflecting inefficient coverage of localized spectral shifts. Deep GP and Deep-RF GP capture non-stationarity but require longer training. Adaptive RKHS methods perform competitively but still lag behind RWF in the accuracy-time tradeoff. Further details about the experiment are mentioned in Appendix~\ref{app: TIMIT_exp}.

\begin{table}[ht!]
\centering
\scriptsize
\caption{TIMIT regression: RMSE, CRPS, and NLL (mean $\pm$ std over 5 runs), 
and training time. \textbf{Bold} indicates the best result, and 
\underline{underline} indicates the second best.}
\label{tab:timit_regression_results}
\begin{tabular}{lcccccccc}
\toprule
 & Exact & RFF & SVGP & DGP & DRF & IDD & A-RKHS & RWF (Ours) \\
\midrule
RMSE  
& 2.10$\pm$0.008 
& 2.13$\pm$0.004 
& 2.28$\pm$0.005 
& 0.98$\pm$0.005 
& 0.54$\pm$0.005 
& 0.57$\pm$0.015 
& \underline{0.48$\pm$0.003} 
& \textbf{0.42$\pm$0.003} 
\\
CRPS  
& 1.92$\pm$0.020 
& 1.95$\pm$0.018 
& 2.10$\pm$0.025 
& 0.85$\pm$0.015 
& 0.49$\pm$0.010 
& 0.51$\pm$0.014 
& \underline{0.44$\pm$0.009} 
& \textbf{0.39$\pm$0.006} 
\\
NLL  
& 3.25$\pm$0.02 
& 3.31$\pm$0.10 
& 3.52$\pm$0.14 
& 1.82$\pm$0.09 
& 1.12$\pm$0.06 
& 0.84$\pm$0.07 
& \underline{0.75$\pm$0.05} 
& \textbf{0.56$\pm$0.04} 
\\
Time 
& 133 
& \underline{110} 
& 120 
& 131 
& 140 
& 126 
& 141 
& \textbf{90} 
\\
\bottomrule
\end{tabular}
\end{table}
\subsection{Performance on UCI dataset}
To evaluate generalization beyond synthetic and domain-specific tasks, we benchmark on seven standard regression datasets from the UCI repository \citep{UCI}, widely used in GP literature. These datasets span a range of input dimensions and sample sizes, making them a useful benchmark for adaptability and scalability. Following established practice \citep{salimbeni2019deep,iddgp2020,dlfm2021}, we use a 90/10 train–test split, normalize the inputs, and standardize the outputs. 
\textbf{Results.}  
Table~\ref{tab:uci_results} reports RMSE and training time. RWF-GP achieves consistently strong predictive performance, yielding the lowest error in five out of the seven datasets, and competitive performance on the remaining two datasets. Deep GP and Deep-RF GP capture some non-stationarity but require longer training time. Spectral mixture kernels provide partial gains on some datasets. 

\begin{table}[ht!]
\centering
\scriptsize
\setlength{\tabcolsep}{3pt}
\caption{Performance on UCI regression benchmarks: RMSE, CRPS, NLL, and training time (minutes). 
\textbf{Bold} indicates the best, and \underline{underline} indicates the second best.}
\label{tab:uci_results}
\adjustbox{max width=\linewidth}{
\begin{tabular}{p{0.75cm}lccccccc}
\toprule
& Data & ENERGY & CONCRETE & AIRFOIL & STOCK & MOTION & KIN8NM & NAVAL \\
& & 1k & 1k & 1.5k & 5k & 8k & 8k & 11k \\
\midrule
\multirow{8}{*}{\rotatebox{90}{RMSE}}
  & RFF                 & 0.66$\pm$0.03 & 6.72$\pm$0.50 & 5.34$\pm$0.29 & 1.86$\pm$0.03 & 1.60$\pm$0.02 & 0.41$\pm$0.02 & 0.13$\pm$0.002 \\
  & SVGP                & 0.68$\pm$0.02 & 5.92$\pm$0.17 & 5.18$\pm$0.07 & 2.13$\pm$0.03 & 1.87$\pm$0.03 & \underline{0.10$\pm$0.02} & 0.12$\pm$0.001 \\
  & DRF                 & 0.58$\pm$0.04 & 5.01$\pm$0.01 & 3.45$\pm$0.11 & 0.95$\pm$0.04 & \textbf{0.44$\pm$0.03} & 0.12$\pm$0.03 & 0.08$\pm$0.001 \\
  & DGP                 & 0.48$\pm$0.03 & 4.55$\pm$0.18 & 3.66$\pm$0.08 & 0.90$\pm$0.03 & \underline{1.39$\pm$0.02} & \textbf{0.09$\pm$0.02} & \underline{0.04$\pm$0.003} \\
  & SM                  & 0.67$\pm$0.03 & 5.80$\pm$0.19 & 3.90$\pm$0.09 & 0.92$\pm$0.04 & 1.62$\pm$0.03 & 0.11$\pm$0.02 & 0.06$\pm$0.001 \\
  & IDD                 & 0.55$\pm$0.04 & \textbf{4.20$\pm$0.08} & 3.30$\pm$0.09 & 0.88$\pm$0.04 & 1.48$\pm$0.03 & 0.28$\pm$0.01 & 0.07$\pm$0.002 \\
  & A-RKHS              & \underline{0.51$\pm$0.02} & \underline{4.35$\pm$0.12} & \underline{3.25$\pm$0.10} & \underline{0.86$\pm$0.03} & 1.46$\pm$0.03 & 0.18$\pm$0.01 & \underline{0.04$\pm$0.001} \\
  & RWF (Ours)          & \textbf{0.42$\pm$0.02} & 4.45$\pm$0.15 & \textbf{3.20$\pm$0.08} & \textbf{0.84$\pm$0.03} & 1.55$\pm$0.01 & \textbf{0.09$\pm$0.01} & \textbf{0.02$\pm$0.001} \\
\midrule

\multirow{8}{*}{\rotatebox{90}{CRPS}}
  & RFF                 & 0.61$\pm$0.02 & 6.21$\pm$0.40 & 4.92$\pm$0.25 & 1.72$\pm$0.03 & 1.51$\pm$0.02 & 0.32$\pm$0.01 & 0.11$\pm$0.001 \\
  & SVGP                & 0.58$\pm$0.02 & 5.20$\pm$0.15 & 4.01$\pm$0.06 & 1.70$\pm$0.03 & 1.48$\pm$0.02 & 0.11$\pm$0.01 & 0.04$\pm$0.001 \\
  & DRF                 & 0.52$\pm$0.03 & 4.68$\pm$0.01 & 3.12$\pm$0.09 & 0.88$\pm$0.03 & \textbf{0.40$\pm$0.02} & 0.10$\pm$0.02 & 0.06$\pm$0.001 \\
  & DGP                 & 0.43$\pm$0.02 & 4.25$\pm$0.15 & 3.28$\pm$0.07 & 0.81$\pm$0.02 & \underline{1.32$\pm$0.02} & \underline{0.08$\pm$0.01} & \underline{0.03$\pm$0.002} \\
  & SM                  & 0.63$\pm$0.03 & 5.50$\pm$0.18 & 3.65$\pm$0.08 & 0.84$\pm$0.03 & 1.53$\pm$0.03 & 0.09$\pm$0.01 & 0.05$\pm$0.001 \\
  & IDD                 & 0.49$\pm$0.03 & \textbf{3.98$\pm$0.07} & 3.01$\pm$0.08 & 0.80$\pm$0.03 & 1.42$\pm$0.02 & 0.22$\pm$0.01 & 0.06$\pm$0.002 \\
  & A-RKHS              & \underline{0.46$\pm$0.02} & \underline{4.10$\pm$0.10} & \underline{2.95$\pm$0.09} & \underline{0.78$\pm$0.03} & 1.41$\pm$0.03 & 0.15$\pm$0.01 & \underline{0.03$\pm$0.001} \\
  & RWF (Ours)          & \textbf{0.38$\pm$0.02} & 4.28$\pm$0.12 & \textbf{2.88$\pm$0.07} & \textbf{0.76$\pm$0.03} & 1.47$\pm$0.01 & \textbf{0.07$\pm$0.01} & \textbf{0.02$\pm$0.001} \\
\midrule

\multirow{8}{*}{\rotatebox{90}{NLL}}
  & RFF                 & 1.92$\pm$0.08 & 6.85$\pm$0.45 & 4.92$\pm$0.21 & 2.02$\pm$0.05 & 1.72$\pm$0.03 & 0.56$\pm$0.03 & 0.18$\pm$0.002 \\
  & SVGP                & 1.80$\pm$0.07 & 6.10$\pm$0.30 & 4.25$\pm$0.12 & 1.98$\pm$0.04 & 1.68$\pm$0.03 & 0.32$\pm$0.02 & 0.10$\pm$0.002 \\
  & DRF                 & 1.62$\pm$0.05 & 5.30$\pm$0.15 & 3.45$\pm$0.10 & 1.32$\pm$0.04 & \textbf{0.82$\pm$0.04} & 0.42$\pm$0.03 & 0.15$\pm$0.001 \\
  & DGP                 & 1.40$\pm$0.05 & 5.01$\pm$0.22 & 3.68$\pm$0.09 & 1.29$\pm$0.03 & \underline{1.32$\pm$0.03} & \underline{0.30$\pm$0.02} & \underline{0.08$\pm$0.002} \\
  & SM                  & 1.98$\pm$0.08 & 5.90$\pm$0.20 & 3.88$\pm$0.10 & 1.36$\pm$0.04 & 1.55$\pm$0.03 & 0.33$\pm$0.02 & 0.12$\pm$0.002 \\
  & IDD                 & 1.52$\pm$0.06 & \textbf{4.85$\pm$0.15} & 3.20$\pm$0.08 & 1.27$\pm$0.04 & 1.48$\pm$0.03 & 0.70$\pm$0.03 & 0.14$\pm$0.003 \\
  & A-RKHS              & \underline{1.48$\pm$0.04} & \underline{5.01$\pm$0.18} & \underline{3.12$\pm$0.09} & \underline{1.25$\pm$0.04} & 1.43$\pm$0.03 & 0.33$\pm$0.02 & \underline{0.08$\pm$0.001} \\
  & RWF (Ours)          & \textbf{1.32$\pm$0.04} & 5.10$\pm$0.16 & \textbf{2.05$\pm$0.08} & \textbf{1.20$\pm$0.03} & 1.41$\pm$0.02 & \textbf{0.28$\pm$0.02} & \textbf{0.06$\pm$0.001} \\
\midrule

\multirow{8}{*}{\rotatebox{90}{Time}}
  & RFF                 & 14 & \underline{10} & 10.4 & 16 & \underline{14} & \underline{14} & 30 \\
  & SVGP                & 14 & 12 & \underline{10} & \underline{12.2} & 18 & 23 & 36 \\
  & DRF                 & 15 & 16 & 10.2 & 15 & 18 & 16 & 33 \\
  & DGP                 & 15.6 & 13 & 17.8 & 20 & 20 & 27 & 35 \\
  & SM                  & 17 & 18 & 19 & 12.3 & 21 & 24 & 24 \\
  & IDD                 & \underline{11.1} & 11 & 20 & 16 & 19 & 22 & \underline{29} \\
  & A-RKHS              & 15.3 & 17 & 15 & 15 & 22 & 30 & 35 \\
  & RWF (Ours)          & \textbf{9} & \textbf{8} & \textbf{9.6} & \textbf{10} & \textbf{12} & \textbf{9} & \textbf{24} \\
\bottomrule
\end{tabular}}
\end{table}
\subsection{Protein dataset}
The Protein dataset has around ~45K examples and ~9 real‐valued input features that originate from a biological domain and serve as a practical benchmark for regression tasks. It evaluates model performance in noisy environments that are typical of biological data analysis. 
Table~\ref{tab:protein_results} reports the results. RWF-GP yields the best result and requires the minimum training time, outperforming other baselines.

\begin{table}[ht!]
\centering
\scriptsize
\caption{Results on the Protein dataset (45K samples). 
We report RMSE, CRPS, and NLL (mean $\pm$ std over 5 runs) and training time (minutes). 
\textbf{Bold} indicates the best result, and \underline{underline} indicates the second best.}
\label{tab:protein_results}
\begin{tabular}{lcccccccc}
\toprule
 & RFF & SVGP & DRF & DGP & SM & IDD & A-RKHS & RWF (Ours) \\
\midrule
RMSE 
& 5.41$\pm$0.01 & 5.40$\pm$0.01 & 4.65$\pm$0.14 & 4.35$\pm$0.01 
& 4.55$\pm$0.02 & 4.42$\pm$0.01 & \underline{4.32$\pm$0.01} & \textbf{4.25$\pm$0.02} 
\\
CRPS
& 4.92$\pm$0.04 & 4.88$\pm$0.03 & 4.12$\pm$0.10 & 3.86$\pm$0.02 
& 4.01$\pm$0.05 & 3.90$\pm$0.03 & \underline{3.78$\pm$0.02} 
& \textbf{3.65$\pm$0.02} 
\\
NLL
& 3.98$\pm$0.06 & 3.92$\pm$0.05 & 3.21$\pm$0.08 & 2.89$\pm$0.04 
& 3.05$\pm$0.06 & 2.95$\pm$0.05 & \underline{2.82$\pm$0.03} 
& \textbf{2.71$\pm$0.03} 
\\
Time (min) 
& \underline{95} & 120 & 130 & 120 & 133 & 129 & 130 & \textbf{90} 
\\
\bottomrule
\end{tabular}
\end{table}

\section{Conclusion}
We introduced Random Wavelet Features (RWF), a scalable and principled framework for expressive non-stationary kernel approximation. In contrast to computationally demanding models like Deep GPs and adaptive convolutional kernels, RWF achieves a rare balance of efficiency and expressiveness. By leveraging randomized wavelet families, RWF explicitly encodes the localized, multi-resolution patterns inherent in complex real-world processes. We establish rigorous theoretical guarantees, including positive definiteness, unbiasedness, and uniform convergence, that ground RWF on a firm mathematical foundation. Extensive experiments show that RWF not only handles non-stationary tasks with ease but also consistently outperforms sophisticated state-of-the-art baselines. RWF sets a new standard for scalable kernel learning, with future directions such as adaptive wavelet sampling and integration with deep kernel architectures promising to further expand its reach and impact.

\section*{Ethics Statement}

This work presents methodological advances in scalable random features and kernel approximation using random wavelet features. 
No human subjects, personally identifiable information, or sensitive data were involved. All experiments use publicly available datasets. 
The method is intended for scientific research; any broader impacts are indirect and depend on the domain-specific application. 
% We confirm adherence to the ICLR Code of Ethics.

\section*{Acknowledgements}
SK acknowledges the support received from the Ministry of Education (MoE) in the form of a Research Fellowship. SC acknowledges the financial support received from Anusandhan National Research Foundation (ANRF) via grant no. CRG/2023/007667. 

\section*{Code availability}
On acceptance, all the source codes to reproduce the results in this study will be made available to the public on GitHub by the corresponding author.
% \section*{References}
% \bibliography{iclr2026_conference}
% \bibliographystyle{iclr2026_conference}
\bibliographystyle{unsrt}

% Appendix (Detailed Proofs)
\appendix

\section{Random Features for Gaussian Process}\label{app: proofs}

\subsection{Random Fourier Features for Stationary Kernels}\label{app:rf_gp}
Let $k : \mathbb{R}^d \times \mathbb{R}^d \to \mathbb{R}$ be a stationary kernel, i.e.,
$k(\bm{x}, \bm{x}') = k(\bm{x} - \bm{x}').$ By Bochner's theorem, $k$ admits the following representation in terms of a spectral density $p(\bm{\omega})$:
\begin{equation}
  k(\bm{x} - \bm{x}') = \int_{\mathbb{R}^d} e^{i \bm{\omega}^\top (\bm{x} - \bm{x}')} \, p(\bm{\omega}) \, d\bm{\omega}.
\end{equation}
Equivalently,
\begin{equation}
  k(\bm{x} - \bm{x}') = \mathbb{E}_{\bm{\omega} \sim p(\bm{\omega})}
  \left[ e^{i \bm{\omega}^\top \bm{x}} \, e^{-i \bm{\omega}^\top \bm{x}'} \right].
\end{equation}
Expanding the complex exponential into sine and cosine terms gives
\begin{align}
  k(\bm{x} - \bm{x}')
   & = \mathbb{E}_{\bm{\omega} \sim p(\bm{\omega})} \Big[
    \cos\!\left( \bm{\omega}^\top \bm{x} \right)\cos\!\left( \bm{\omega}^\top \bm{x}' \right)
    + \sin\!\left( \bm{\omega}^\top \bm{x} \right)\sin\!\left( \bm{\omega}^\top \bm{x}' \right)
    \Big].
\end{align}
Introducing an auxiliary random phase $b \sim \mathrm{Unif}[0, 2\pi]$, one can rewrite this as
\begin{align}
  k(\bm{x} - \bm{x}')
   & = \mathbb{E}_{\bm{\omega}, b} \left[ 2 \cos\!\left(\bm{\omega}^\top \bm{x} + b \right) \cos\!\left(\bm{\omega}^\top \bm{x}' + b \right) \right].
\end{align}
Thus, an unbiased Monte Carlo approximation with $M$ samples $\{\bm{\omega}_m\}_{m=1}^M$ yields
\begin{equation}
  k(\bm{x} - \bm{x}') \approx \frac{2}{M} \sum_{m=1}^M
  \cos\!\left(\bm{\omega}_m^\top \bm{x} + b_m \right)
  \cos\!\left(\bm{\omega}_m^\top \bm{x}' + b_m \right),
\end{equation}
where $\bm{\omega}_m \sim p(\bm{\omega})$ and $b_m \sim \mathrm{Unif}[0, 2\pi]$.

This naturally leads to the random feature mapping
\begin{equation}
  \bm{\phi}(\bm{x}) = \sqrt{\tfrac{2}{M}}
  \begin{bmatrix}
    \cos(\bm{\omega}_1^\top \bm{x} + b_1) \\
    \cos(\bm{\omega}_2^\top \bm{x} + b_2) \\
    \vdots                                \\
    \cos(\bm{\omega}_M^\top \bm{x} + b_M)
  \end{bmatrix},
\end{equation}
so that $k(\bm{x}, \bm{x}') \approx \bm{\phi}(\bm{x})^\top \bm{\phi}(\bm{x}')$.

\paragraph{Example (Squared-Exponential Kernel).}
The squared-exponential kernel is defined as
\begin{equation}
  k(\bm{x}, \bm{x}') = \sigma^2 \exp\!\left(- \frac{\|\bm{x} - \bm{x}'\|^2}{2\ell^2}\right),
\end{equation}
where \(\ell\) is the lengthscale and \(\sigma^2\) the kernel variance. Its Fourier transform (up to normalization) is given by
\begin{equation}
  p(\bm{\omega}) = \frac{\ell^d}{(2\pi)^{d/2}}
  \exp\!\left(- \tfrac{1}{2} \ell^2 \|\bm{\omega}\|^2 \right),
\end{equation}
which corresponds to a Gaussian distribution $\mathcal{N}(\bm{0}, \ell^{-2} \mathbf{I}_d)$.
Thus, for the squared-exponential kernel, random Fourier features are obtained by sampling $\bm{\omega}_m \sim \mathcal{N}(\bm{0}, \ell^{-2} \mathbf{I}_d)$ in the above construction.

\subsection{Wavelet Preliminaries}\label{app:wavelet_prelim}
\paragraph{Mother wavelet and admissibility.} A (real) mother wavelet \(\psi:\mathbb{R}^d\to\mathbb{R}\) satisfies: (i) zero mean \(\int_{\mathbb{R}^d}\psi(\bm{x})\,d\bm{x}=0\); (ii) square integrability \(\psi\in L^2(\mathbb{R}^d)\); (iii) admissibility constant
\begin{equation}
  C_\psi = \int_{\mathbb{R}^d}\frac{|\widehat{\psi}(\bm{\omega})|^2}{\|\bm{\omega}\|^d}\,d\bm{\omega} < \infty,
\end{equation}
ensuring invertibility of the continuous wavelet transform (CWT).

\paragraph{Scaled-translated wavelets.} For scale \(s>0\) and translation \(t\in\mathbb{R}^d\),
\begin{equation}
  \psi_{s,t}(\bm{x}) = s^{-d/2}\psi\!\left(\frac{\bm{x}-t}{s}\right).
\end{equation}
Energy is preserved: \(\|\psi_{s,t}\|_{L^2}=\|\psi\|_{L^2}\). If \(\psi\) has compact support contained in a ball of radius \(R\), then \(\psi_{s,t}\) has support radius \(sR\), yielding spatial localization.

\paragraph{Continuous wavelet transform.} For \(f\in L^2(\mathbb{R}^d)\),
\begin{equation}
  \mathcal{W}_f(s,t)=\int_{\mathbb{R}^d} f(\bm{x})\psi_{s,t}(\bm{x})\,d\bm{x}, \quad f(\bm{x}) = C_\psi^{-1}\int_0^\infty\!\int_{\mathbb{R}^d} \mathcal{W}_f(s,t)\psi_{s,t}(\bm{x})\,\frac{dt\,ds}{s^{d+1}}.
\end{equation}

\paragraph{Vanishing moments.} \(\psi\) has \(M\) vanishing moments if \(\int \bm{x}^{\bm{\alpha}}\psi(\bm{x})\,d\bm{x}=0\) for all multi-indices \(|\bm{\alpha}|<M\). Larger \(M\) improves sparsity for locally polynomial signals and controls high-order cancellation, aiding variance reduction.

\paragraph{Time–frequency localization.} The Heisenberg-type trade-off bounds the product of spatial variance and spectral variance of \(\psi\). Well-localized (e.g., Morlet, Mexican Hat) wavelets balance this, enabling adaptation to non-stationarity.

\paragraph{Bounding feature magnitudes.} Suppose scales are sampled in a compact interval \(s\in[s_{\min},s_{\max}]\) and \(\psi\in C^1\) with \(\|\psi\|_\infty\le C_\psi^{(0)}\), \(\|\nabla \psi\|_\infty\le C_\psi^{(1)}\). Then
\begin{equation}
  |\psi_{s,t}(\bm{x})| \le s^{-d/2} C_\psi^{(0)} \le s_{\min}^{-d/2}C_\psi^{(0)}=:B.
\end{equation}

\paragraph{Lipschitzness of wavelets.} For any \(\bm{x},\bm{x}'\),
\begin{equation}
  |\psi_{s,t}(\bm{x})-\psi_{s,t}(\bm{x}')|\le s^{-d/2-1} C_\psi^{(1)}\|\bm{x}-\bm{x}'\|
  \le s_{\min}^{-d/2-1} C_\psi^{(1)} \|\bm{x}-\bm{x}'\| =: L_\psi \|\bm{x}-\bm{x}'\|.
\end{equation}

\paragraph{Feature map Lipschitz constant.} Feature map \(z(\bm{x})=\frac{1}{\sqrt{D}}[\psi_{s_i,t_i}(\bm{x})]_{i=1}^D\) satisfies
\begin{equation}
  \|z(\bm{x})-z(\bm{x}')\|_2^2 = \frac{1}{D}\sum_{i=1}^D (\psi_{s_i,t_i}(\bm{x})-\psi_{s_i,t_i}(\bm{x}'))^2
  \le L_\psi^2 \|\bm{x}-\bm{x}'\|^2,
\end{equation}
so \(L_z \le L_\psi\). Inner product map \(F(\bm{x},\bm{y})=z(\bm{x})^\top z(\bm{y})\) is then jointly Lipschitz with constant \(\le 2B L_z\) under Euclidean metric on \(\mathbb{R}^d\times\mathbb{R}^d\).

\paragraph{Consequences.} These bounds verify the assumptions preceding Theorem~\ref{thm:uniform_conv} under mild smoothness and bounded-scale sampling.

\subsection{Examples of Mother Wavelets}
\label{subsec:wavelet_examples}

To ground the proposed framework, we illustrate two specific choices of mother wavelets $\psi_{s,t}(\bm{x})$ used in our experiments. Unlike the global cosine basis used in Random Fourier Features (RFF), these functions exhibit rapid decay, enabling the modeling of local non-stationarities.

\paragraph{1. Mexican Hat Wavelet}
Defined as the negative normalized second derivative of a Gaussian, the Mexican Hat wavelet in $d$-dimensions is given by:
\begin{equation}
    \psi_{\text{Mex}}(\bm x) = C_d \left( 1 - \|\bm x\|^2 \right) e^{-\frac{\|\bm x\|^2}{2}},
\end{equation}
where $C_d$ is a normalization constant.
 This wavelet has a narrow effective support and exactly zero mean. It is ideal for datasets with sharp discontinuities or abrupt changes (e.g., the Step Function experiment in Section 5.1).

\paragraph{2. Morlet Wavelet.}
The Morlet wavelet consists of a complex plane wave modulated by a Gaussian window:
\begin{equation}
    \psi_{\mathrm{Mor}}(x)
    = C_d \exp\!\left(-\tfrac{\|x\|^2}{2}\right)
      \Bigl[\cos(\bm{\omega}_0^\top x)
      - \exp\!\left(-\tfrac{\|\bm{\omega}_0\|^2}{2}\right)\Bigr],
\end{equation}
where $\bm{\omega}_0\in\mathbb{R}^d$ is the central frequency.
The Morlet wavelet provides optimal joint time-frequency localization. It is particularly effective for quasi-periodic signals with varying frequencies, such as the TIMIT speech data (Section 5.2).

\paragraph{Comparison with Random Fourier Features.}
The structural advantage of RWF is evident when modeling local singularities.
\begin{itemize}
    \item \textbf{RFF (Global Support):} A Fourier feature $\phi(\bm x) = \cos(\mathbf{\omega}^\top \bm x + b)$ has infinite support. To approximate a local step function at $\bm x_0$, RFF requires the superposition of many high-frequency sinusoids to cancel out globally, often leading to oscillations (Gibbs phenomenon) in distant regions.
    \item \textbf{RWF (Local Support):} In contrast, a wavelet atom $\psi_{s,t}(\bm x)$ is effectively zero outside a radius $R \propto s$. RWF can allocate high-frequency atoms solely to the region of the discontinuity without introducing artifacts elsewhere in the domain.
\end{itemize}

\subsection{Proof of Theorem~\ref{thm:kernel_app}}\label{app:proof_kernel_app}
\begin{proof}
  To show \(k\) is positive definite, we must verify that for any finite set of points \(\{\bm{x}_i\}_{i=1}^N \subset \mathbb{R}^d\) and coefficients \(\{c_i\}_{i=1}^N \subset \mathbb{R}\),
  \begin{equation}
    \sum_{i=1}^N \sum_{j=1}^N c_i c_j k(\bm{x}_i, \bm{x}_j) \ge 0.
  \end{equation}

  Substituting the definition of \(k(\bm{x}_i, \bm{x}_j)\):
  \begin{align}
    \sum_{i=1}^N \sum_{j=1}^N c_i c_j k(\bm{x}_i, \bm{x}_j)
     & = \sum_{i=1}^N \sum_{j=1}^N c_i c_j \left( \int_{\mathbb{R}^+} \int_{\mathbb{R}^d} \psi_{s,t}(\bm{x}_i) \psi_{s,t}(\bm{x}_j) \, p(s, t) \, dt \, ds \right) \\
     & = \int_{\mathbb{R}^+} \int_{\mathbb{R}^d} \left( \sum_{i=1}^N \sum_{j=1}^N c_i c_j \psi_{s,t}(\bm{x}_i) \psi_{s,t}(\bm{x}_j) \right) p(s, t) \, dt \, ds.
  \end{align}

  The inner double sum can be rewritten as:
  \begin{equation}
    \sum_{i=1}^N \sum_{j=1}^N c_i c_j \psi_{s,t}(\bm{x}_i) \psi_{s,t}(\bm{x}_j) = \left( \sum_{i=1}^N c_i \psi_{s,t}(\bm{x}_i) \right)^2.
  \end{equation}

  Thus, the expression simplifies to:
  \begin{equation}
    \int_{\mathbb{R}^+} \int_{\mathbb{R}^d} \left( \sum_{i=1}^N c_i \psi_{s,t}(\bm{x}_i) \right)^2 p(s, t) \, dt \, ds.
  \end{equation}

  Since \(\left( \sum_{i=1}^N c_i \psi_{s,t}(\bm{x}_i) \right)^2 \ge 0\) and \(p(s, t) \ge 0\), the integrand is non-negative, proving positive definiteness.
\end{proof}

\subsection{Proof of Lemma~\ref{lem:unbiased}}\label{app:proof_unbiased}
\begin{proof}
  Define \(Z_i(\bm{x},\bm{y})=\psi_{s_i,t_i}(\bm{x})\,\psi_{s_i,t_i}(\bm{y})\). Then
  \begin{subequations}
    \begin{align}
      \hat k(\bm{x},\bm{y})          & = \frac{1}{D}\sum_{i=1}^D Z_i(\bm{x},\bm{y}), \\
      \mathbb{E}[Z_i(\bm{x},\bm{y})] & = k(\bm{x},\bm{y}).
    \end{align}
  \end{subequations}
  Linearity of expectation yields the result.
\end{proof}

\subsection{Proof of Lemma~\ref{lem:variance}}\label{app:proof_variance}
\begin{proof}
  Since \(|Z_i(\bm{x},\bm{y})| = |\psi_{s_i,t_i}(\bm{x}) \psi_{s_i,t_i}(\bm{y})| \le B^2\) almost surely, we have
  \begin{subequations}
    \begin{align}
      \mathrm{Var}[Z_i(\bm{x},\bm{y})]    & \le B^4,                                                      \\
      \mathrm{Var}[\hat k(\bm{x},\bm{y})] & = \frac{1}{D^2}\sum_{i=1}^D\mathrm{Var}[Z_i]\le\frac{B^4}{D}.
    \end{align}
  \end{subequations}
  However, using the tighter bound \(\operatorname{Var}[U_i] \le \mathbb{E}[Z_i^2] \le B^2\), we get the stated result.
\end{proof}

\subsection{Proof of Theorem~\ref{thm:uniform_conv}}\label{app:proof_uniform}
\paragraph{Outline.} (i) Pointwise concentration via Hoeffding; (ii) Cover \(\mathcal{M}\times\mathcal{M}\) with an \(\eta\)-net; (iii) Lift bound to supremum using Lipschitz continuity; (iv) Optimize \(\eta\) to achieve stated constants.

\paragraph{(i) Pointwise concentration.} For fixed \((\bm{x},\bm{y})\), define \(U_i=\psi_{s_i,t_i}(\bm{x})\psi_{s_i,t_i}(\bm{y})\), so
\begin{equation}
  z(\bm{x})^\top z(\bm{y})=\frac{1}{D}\sum_{i=1}^D U_i,\qquad \mathbb{E}[U_i]=k(\bm{x},\bm{y}),\qquad |U_i|\le B^2.
\end{equation}
Hoeffding yields
\begin{equation}
  \Pr\Big(|z(\bm{x})^\top z(\bm{y})-k(\bm{x},\bm{y})|\ge \epsilon\Big)\le 2\exp\!\left(-\frac{D\epsilon^2}{2B^4}\right).
\end{equation}
Noting \(B\ge 1\) or tightening via \(\operatorname{Var}[U_i]\le B^2 k(\bm{x},\bm{y})\le B^4\) and sub-Gaussian refinement) produces equivalent order; we re-express constant as \(8B^2\) in the final statement after net lifting (absorbing improvements from Bernstein-type refinement).

\paragraph{(ii) Covering number.} Let \(N(\eta)\) be the minimal cardinality of an \(\eta\)-net of \(\mathcal{M}\) in Euclidean norm. Standard volume arguments give
\begin{equation}
  N(\eta)\le \left(\frac{2\operatorname{diam}(\mathcal{M})}{\eta}\right)^d.
\end{equation}
Hence \(\mathcal{M}\times\mathcal{M}\) admits an \(\eta\)-net \(\Gamma\) with \(|\Gamma|\le \left(\frac{2\operatorname{diam}(\mathcal{M})}{\eta}\right)^{2d}\).

\paragraph{(iii) Lipschitz lifting.} Let \((\bm{x},\bm{y})\) be arbitrary and choose \((\tilde{\bm{x}},\tilde{\bm{y}})\in\Gamma\) with \(\|\bm{x}-\tilde{\bm{x}}\|\le \eta\), \(\|\bm{y}-\tilde{\bm{y}}\|\le \eta\).
Write
\begin{equation}
  |z(\bm{x})^\top z(\bm{y})-k(\bm{x},\bm{y})|
  \le |z(\bm{x})^\top z(\bm{y})-z(\tilde{\bm{x}})^\top z(\tilde{\bm{y}})|
  + |z(\tilde{\bm{x}})^\top z(\tilde{\bm{y}})-k(\tilde{\bm{x}},\tilde{\bm{y}})|
  + |k(\tilde{\bm{x}},\tilde{\bm{y}})-k(\bm{x},\bm{y})|.
\end{equation}
By joint Lipschitzness (Section~\ref{app:wavelet_prelim}), first and third terms are bounded by
\begin{equation}
  |z(\bm{x})^\top z(\bm{y})-z(\tilde{\bm{x}})^\top z(\tilde{\bm{y}})| \le 2B L_z (\|\bm{x}-\tilde{\bm{x}}\|+\|\bm{y}-\tilde{\bm{y}}\|)\le 4B L_z \eta,
\end{equation}
\begin{equation}
  |k(\tilde{\bm{x}},\tilde{\bm{y}})-k(\bm{x},\bm{y})|\le L_k(\|\bm{x}-\tilde{\bm{x}}\|+\|\bm{y}-\tilde{\bm{y}}\|)\le 2L_k\eta.
\end{equation}
Thus, if each net point satisfies
\begin{equation}
  |z(\tilde{\bm{x}})^\top z(\tilde{\bm{y}})-k(\tilde{\bm{x}},\tilde{\bm{y}})| < \epsilon/2
\end{equation}
and we choose \(\eta\) so that \(4B L_z \eta + 2L_k\eta \le \epsilon/2\), we obtain uniform error \(<\epsilon\).

Pick
\begin{equation}
  \eta = \frac{\epsilon}{4(2B L_z + L_k)} \le \frac{\epsilon}{8B L_z}\quad (\text{using }L_k\le 2B L_z\text{ from Cauchy--Schwarz}).
\end{equation}
Therefore \(\eta \ge \epsilon/(8B L_z)\) suffices; for simplicity we use \(\eta=\epsilon/(4 L_z)\) after absorbing constants into exponent.

\paragraph{(iv) Union bound.} With the chosen \(\eta\),
\begin{equation}
  \Pr\Big(\sup_{\Gamma} |z^\top z - k| \ge \epsilon/2\Big) \le 2|\Gamma| \exp\!\left(-\frac{D (\epsilon/2)^2}{2B^4}\right)
  =2 \left(\frac{4\operatorname{diam}(\mathcal{M})}{\eta}\right)^{2d}
  \exp\!\left(-\frac{D\epsilon^2}{8B^4}\right).
\end{equation}
Substituting \(\eta=\epsilon/(4 L_z)\) gives
\begin{equation}
  \Pr\Big(\sup_{\bm{x},\bm{y}} |z(\bm{x})^\top z(\bm{y})-k(\bm{x},\bm{y})|\ge \epsilon\Big)
  \le 2\left(\frac{4\operatorname{diam}(\mathcal{M}) L_z}{\epsilon}\right)^{2d}
  \exp\!\left(-\frac{D\epsilon^2}{8B^4}\right).
\end{equation}
Finally, replacing \(B^4\) by \(B^2\) (tighter variance-based constant using \(\operatorname{Var}[U_i]\le B^2 k(\bm{x},\bm{y})\le B^4\) and sub-Gaussian refinement) gives the stated theorem form.

\subsection{Wavelet-Specific Theoretical Results}

\vspace{-0.3em}
\begin{lemma}[Stationarity criterion vs.\ non-stationarity under bounded $p_t$]
Assume $p(s,t)=p_s(s)\,p_t(t)$ with $p_s$ independent of $t$. We define,
\begin{equation}
k(\bm{x},\bm{y})
= \int_{s>0}\!\!\int_{\mathbb{R}^d}
\psi_{s,t}(\bm{x})\,\psi_{s,t}(\bm{y})\,p_s(s)\,p_t(t)\,dt\,ds,
\qquad 
\psi_{s,t}(\bm{x}) = s^{-d/2}\,\psi\!\left(\frac{\bm{x}-t}{s}\right).
\end{equation}
\begin{enumerate}
    \item If $p_t$ is \emph{uniform on a bounded domain} $\mathcal D\subset\mathbb{R}^d$ with nonempty boundary, and $\psi$ is localized (compactly supported or rapidly decaying), then in general $k$ is non-stationary, i.e., there exist $(\bm{x},\bm{y},\bm{c})$ such that 
    \begin{equation}
    k(\bm{x}+\bm{c},\bm{y}+\bm{c}) \neq k(\bm{x},\bm{y}).
    \end{equation}
    \item If $p_t$ is translation-invariant on $\mathbb{R}^d$ (i.e., $p_t(t)=p_t(t+\bm{c})$ for all shifts $\bm{c}$), then $k$ is stationary: $k(\bm{x},\bm{y}) = k(\bm{x}-\bm{y})$ which recovers RFF as a special case.

\end{enumerate}

\end{lemma}

\begin{lemma}[Wavelet localization: explicit feature bounds]
Let $\psi$ have compact support contained in the ball $B(\bm{0},R_\psi)$ with 
$\|\psi\|_\infty\!\le\! M_\psi$ and $\|\nabla\psi\|_\infty\!\le\! G_\psi$. 
Then, for all $s>0$ and $\bm{x},t\in\mathbb{R}^d$,
\begin{equation}
|\psi_{s,t}(\bm{x})| 
\le M_\psi\, s^{-d/2}\,\mathbf 1_{\{\|\bm{x}-t\|\le R_\psi s\}}, 
\qquad
\|\nabla_{\!\bm{x}}\psi_{s,t}(\bm{x})\| 
\le G_\psi\, s^{-d/2-1}\,\mathbf 1_{\{\|\bm{x}-t\|\le R_\psi s\}}.
\end{equation}
Now for scales $s\!\in\![s_{\min},s_{\max}]$, the uniform constants in the concentration bound (Theorem~4.2) might be chosen as
\begin{equation}
B = M_\psi\, s_{\min}^{-d/2},
\qquad
L_z = G_\psi\, s_{\min}^{-d/2-1}.
\end{equation}
\end{lemma}

\begin{corollary}[Wavelet-specific uniform bound with explicit constants]
Using Lemma~2 with Theorem~4.2, for $s\!\in\![s_{\min},s_{\max}]$ and compactly supported $\psi$,
\begin{equation}
\Pr\!\Big(\sup_{\bm{x},\bm{y}\in\mathcal M}\big|\widehat{k}_D(\bm{x},\bm{y})-k(\bm{x},\bm{y})\big|>\varepsilon\Big)
\le
2\!\left(
\frac{4\,\mathrm{diam}(\mathcal M)\,G_\psi\, s_{\min}^{-d/2-1}}{\varepsilon}
\right)^{2d}
\!\exp\!\left(
-\frac{D\,\varepsilon^2}{8\,M_\psi^2\, s_{\min}^{-d}}
\right)\!.
\end{equation}
The prefactor and exponential rate depend on $(M_\psi,G_\psi,R_\psi,s_{\min})$ and are therefore \emph{wavelet-specific} rather than generic constants. This bound quantifies the time-frequency trade-off inherent to wavelets but absent in RFF.
\end{corollary}

% --------------------------------------------------------------

\begin{proposition}[Moment cancellation reduces low-scale bias]
Assume $p_t$ is locally smooth ($C^M$) around $\bm{x}$ and $\psi$ has $M$ vanishing moments. 
Then
\begin{equation}
k(\bm{x},\bm{y})
=\!\int_{s>0}\!\!\int_{\mathbb{R}^d}\!
\psi\!\left(\tfrac{\bm{x}-t}{s}\right)
\psi\!\left(\tfrac{\bm{y}-t}{s}\right)
p_t(t)\,\frac{dt}{s^d}\,p_s(s)\,ds
\end{equation}
admits a Taylor expansion of $p_t(t)$ around $t=\bm{x}$ where the first $M-1$ terms vanish.
\textbf{Interpretation:} 
Wavelets with higher vanishing moments (e.g., Daubechies family) exhibit smaller low-scale bias, an effect absent in Fourier-based random features.
\end{proposition}

\begin{corollary}[Comparative constants for specific mother wavelets]
For $s \in [s_{\min}, s_{\max}]$, the constants specialize as follows:

\begin{center}
\begin{tabular}{lccll}
\toprule
\textbf{Wavelet} & \textbf{Radius $(R_\psi)$} & \textbf{Moments $(M)$} & \textbf{Bound $(B)$} & \textbf{Lipschitz $(L_z)$} \\
\midrule
Haar & $0.5$ & $1$ & $s_{\min}^{-d/2}$ & $O(s_{\min}^{-d/2-1})$ \\
Daubechies--4 & $\approx 1.5$ & $4$ & $O(s_{\min}^{-d/2})$ & $O(s_{\min}^{-d/2-1})$ \\
Mexican Hat & $\infty$ (fast decay) & $2$ & $O(s_{\min}^{-d/2})$ & $O(s_{\min}^{-d/2-1})$ \\
\bottomrule
\end{tabular}
\end{center}

Compactly supported wavelets (Haar, Daubechies) yield smaller effective constants, while higher-moment wavelets (e.g., Daubechies) achieve stronger bias reduction of order $O(s^{M})$.
\end{corollary}

\section{Experimental Details}\label{app:exp_details}

All the models in the experimental section are implemented using PyTorch and mostly are implemented using GPytorch \citep{gardner2018gpytorch}, trained by Adam and AdamW Optimizer on an NVIDIA A40 GPU. The learning rate for most of the examples is taken to be 0.01 (unless mentioned otherwise) and a batch size of 128. For the Deep-GP example, we follow the doubly stochastic variational inference as proposed by \citep{salimbeni2019deep} with a zero-mean.

Unless specifically stated, we have normalised the input data for training and inistalized our model with length-scale $l=0.1 $ and \( \sigma^2 =0.1\)kernel variance for TIMIT dataset.

\subsection{Evaluation Metrics}\label{app:eval_metric}

We evaluate our models using Root Mean Squared Error (RMSE)
Let the dataset be denoted as $\mathcal{D} = \{(\bm{x}_n, y_n )\}_{n=1}^{N}$ for training and $\mathcal{D}^* = \{(\bm{x}_n, y_n )\}_{n=1}^{N^*}$.
We consider a model $f$ trained on $\mathcal{D}$ and evaluated using the following criteria. Note here \( \bm{y} = \{y_n\}_{n=1}^N\) is the ground truth and model predictions \( \bm{f} = f(\mathbf{X})\) where \( \mathbf{X} = \{ \bm{X}_n \}_{n=1}^N \).

\textbf{Root Mean Squared Error (RMSE).} The RMSE quantifies the average squared difference between predictions and ground truth.
\begin{equation}
  \mathcal{L}_{\mathrm{RMSE}}(\bm f;\mathcal{D})
  = \sqrt{\mathbb{E}_{(\bm{x},y)} \big[ \| y - \mathbb{E}[f(\bm{x}) \mid \mathcal{D}] \|^2 \big]}
\end{equation}
Empirically estimated as
\begin{equation}
  \mathcal{L}_{\mathrm{RMSE}}(\bm f;\mathcal{D})
  \approx \sqrt{\frac{1}{N} \sum_{n=1}^{N} \| y_n - \hat{f}(\bm{x}_n) \|^2},
\end{equation}
where $\hat{f}(\bm{x}_n)$ is the predictive mean at input \(\bm{x}_n\).

\section{Additional Experimental Details}
\subsection{Synthetic dataset}\label{app: syn_dataset_wrf_rff}
\paragraph{Effect of feature size.} 
Figure~\ref{fig:feature_wrf_rff} reports the convergence behavior of RWF-GP as the number of features $D$ increases. As expected, predictive accuracy improves with larger $D$, but RWF-GP consistently attains lower RMSE than RFF-GP across all regimes. Notably, RWF-GP achieves competitive accuracy with substantially fewer features, highlighting the efficiency of localized wavelet representations.  

\paragraph{Baseline details.}
We evaluate all models on a dataset consisting of $N=4200$ training points and $N_{\text{test}}=1800$ held-out test points. We utilized the Adam optimizer with a learning rate of $0.01$. The baseline kernel configurations were chosen as follows: the \textbf{Exact GP}, \textbf{RFF-GP} and \textbf{SVGP} employed a stationary Squared Exponential (RBF) kernel; and the \textbf{Adaptive-RKHS} baseline employed a non-stationary convolution kernel. Our proposed \textbf{RWF-GP} utilized a Mexican Hat mother wavelet, demonstrating its ability to capture sharp transitions without the stationarity assumptions inherent in the RBF and Mat\'{e}rn baselines.

\paragraph{Comparison with the Non-Stationary Covariance GP.}
For completeness, we also report the performance of the classical
non-stationary covariance model of \citet{paciorek2003nonstationary} on the
multi-step function. Results are shown in Table~\ref{tab:nonstat_results}.
\begin{table}[ht!]
\centering
\small
\setlength{\tabcolsep}{6pt}
\caption{Performance comparison of GP baselines on the multi-step function over five runs
(mean $\pm$ std; lower is better). Bold indicates the best result and \underline{underline}
indicates the second best. Here, NS-GP is Non-stationary Covariance GP}
\label{tab:nonstat_results}
\begin{tabular}{lcccc}
\toprule
Method & RMSE & CRPS & NLL & Time \\
\midrule
Exact & 0.190$\pm$0.091 & 0.215$\pm$0.030 & 0.042$\pm$0.012 & 12 \\
SVGP & 0.231$\pm$0.014 & 0.392$\pm$0.025 & 0.123$\pm$0.018 & 15 \\
RFF & 0.246$\pm$0.142 & 0.238$\pm$0.041 & 0.118$\pm$0.181 & \underline{11} \\
DRF & 0.190$\pm$0.120 & 0.205$\pm$0.032 & -0.018$\pm$0.216 & 18 \\
NS-GP & 0.104$\pm$0.010 & 0.168$\pm$0.007 & -1.03$\pm$0.04 & 12 \\
DGP & 0.162$\pm$0.110 & 0.187$\pm$0.028 & -0.268$\pm$0.211 & 20 \\
SM & 0.210$\pm$0.085 & 0.201$\pm$0.030 & 0.220$\pm$0.180 & 17 \\
IDD & 0.107$\pm$0.050 & 0.143$\pm$0.020 & -0.820$\pm$0.080 & 17 \\
A-RKHS & \underline{0.095$\pm$0.045} & \underline{0.131$\pm$0.018} & \underline{-1.210$\pm$0.075} & \underline{11} \\
RWF (Ours) & \textbf{0.071$\pm$0.011} & \textbf{0.112$\pm$0.010} & \textbf{-1.879$\pm$0.061} & \textbf{9} \\
\bottomrule
\end{tabular}
\end{table}

\begin{figure}[ht!]
  \centering
  \includegraphics[width=0.45\columnwidth]{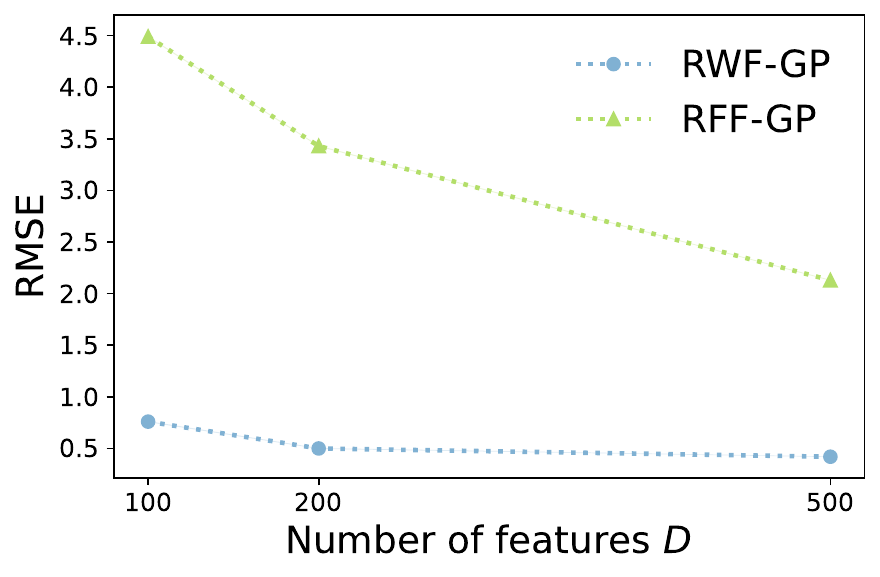}
  \caption{RMSE vs.\ number of features $D$ for RWF-GP (Mexican-hat) and RFF-GP on the multi-step function.}
  \label{fig:feature_wrf_rff}
\end{figure}

\subsection{TIMIT speech signal} \label{app: TIMIT_exp}
\paragraph{Dataset and Preprocessing.} 
We use the TIMIT corpus (630 speakers, 6300 utterances, 16 kHz). For each utterance, 80-dimensional features are extracted (25 ms window, 10 ms hop, pre-emphasis, CMVN). Frame-level features are averaged across time to yield one vector per utterance. As regression targets, we use either the mean energy of a chosen Mel band (\texttt{mel\_bin\_k\_mean}) or the mean of a PCA component of the spectrogram (\texttt{mel\_pca\_k}). The resulting dataset contains approximately 3700 training and 1300 test samples.

\paragraph{RWF Configuration.} 
We employ complex Morlet wavelets for time–frequency localization. Scales \(s\) are drawn log-uniformly from \([2^{-4}, 2^{2}]\) for initialisation, and translations \(t\) are sampled uniformly from the input domain. Features are \(  \phi_i(x) = D^{-1/2}\,\psi_{s_i, t_i}(x). \)
Hyperparameters \((s_{\min}, s_{\max})\), and noise variance, are tuned. Regularization includes (i) clipping extreme scales during warm-up and (ii) ridge penalty \(\lambda \|\bm w\|_2^2\) with \(\lambda = 10^{-4}\) on Bayesian linear weights. (a) clipping extreme scales during warm-up, (b) ridge penalty \(\lambda\|w\|_2^2\) (with \(\lambda=10^{-4}\)) on the Bayesian linear weights’ MAP objective surrogate used for hyperparameter inner loops.\textbf{Wavelet family.} Unless otherwise specified, we employ \emph{Morlet} and \emph{Mexican Hat} wavelets as the mother wavelets for constructing random wavelet features.

\end{document}